\documentclass{article}

\usepackage{microtype}
\usepackage{graphicx}
\usepackage{subfigure}
\usepackage{booktabs} 
\usepackage{algorithm2e}
\usepackage{hyperref}
\usepackage{amstext, amscd, amsthm, amssymb}
\usepackage{amsthm}
\usepackage{thmtools, thm-restate}
\usepackage{thm-patch, thm-kv, thm-autoref}

\theoremstyle{plain}

\theoremstyle{definition}

\theoremstyle{remark}

\usepackage[accepted]{icml2024}

\PassOptionsToPackage{capitalize,noabbrev}{cleveref}

\usepackage[T1]{fontenc}    
\usepackage[utf8]{inputenc} 
\usepackage{hyperref}       
\usepackage{url}            
\usepackage{booktabs}       
\usepackage{nicefrac}       
\usepackage{microtype}      

\usepackage{enumitem}
\usepackage{framed}
\usepackage{perpage}
\usepackage{array,etoolbox}
\usepackage{xspace}
\usepackage[mathscr]{eucal}
\usepackage{mathtools}
\usepackage{bm} 
\usepackage{colors}
\usepackage{local_defs}
\usepackage{amssymb}
\usepackage{multirow}
\usepackage{fontawesome5}

\usepackage[textsize=tiny]{todonotes}

\icmltitlerunning{A Fresh Take on Stale Embeddings: Improving Dense Retriever Training with Corrector Networks}

\begin{document}

\twocolumn[
\icmltitle{A Fresh Take on Stale Embeddings: \\ Improving Dense Retriever Training with Corrector Networks}

\icmlsetsymbol{equal}{*}

\begin{icmlauthorlist}
\icmlauthor{Nicholas Monath}{equal,yyy}
\icmlauthor{Will Grathwohl}{equal,yyy}
\icmlauthor{Michael Boratko}{yyy}
\icmlauthor{Rob Fergus}{yyy}
\icmlauthor{Andrew McCallum}{yyy}
\icmlauthor{Manzil Zaheer}{yyy}
\end{icmlauthorlist}

\icmlaffiliation{yyy}{Google DeepMind}

\icmlcorrespondingauthor{Will Grathwohl}{wgrathwohl@google.com}
\icmlcorrespondingauthor{Nicholas Monath}{nmonath@google.com}

\icmlkeywords{Retrieval, Negative Mining}

\vskip 0.3in
]

\printAffiliationsAndNotice{$^*$Equal Contribution. Order determined by coin flip.} 

\renewcommand{\P}{P}

\def\drm{{\mathrm{d}}}

\newcommand{\DCal}{\mathscr{D}}
\newcommand{\SCal}{\mathscr{S}}

\newcommand{\FC}{\mathcal{F}}
\newcommand{\GC}{\mathcal{G}}
\newcommand{\HC}{\mathcal{H}}
\newcommand{\OC}{\mathcal{O}}
\newcommand{\WC}{\mathcal{W}}

\newcommand{\RF}{\mathfrak{R}}

\newcommand{\targetword}{target\xspace}
\newcommand{\queryword}{query\xspace}
\newcommand{\targetsword}{targets\xspace}
\newcommand{\queriesword}{queries\xspace}

\newcommand{\parameters}{\Theta}
\newcommand{\approxparameters}{\Psi}

\newcommand{\setoftargetstraining}{Y}
\newcommand{\targetencoder}{\ensuremath{g}}
\newcommand{\staletargetencoder}{\ensuremath{\targetencoder'}}
\newcommand{\queryencoder}{\ensuremath{f}}
\newcommand{\approxmodel}{\ensuremath{h}}
\newcommand{\numtargets}{\ensuremath{N}}
\newcommand{\dimensions}{\ensuremath{D}}
\newcommand{\point}{\ensuremath{x}}
\newcommand{\target}{\ensuremath{y}}
\newcommand{\alltargets}{\ensuremath{\mathcal{Y}}}
\newcommand{\logit}{\ensuremath{s}}
\newcommand{\temperature}{\beta}
\newcommand{\partitionfn}{Z}
\newcommand{\modelterminology}{corrector model\xspace}
\newcommand{\reals}{\mathcal{R}}
\newcommand{\buffer}{B}
\newcommand{\ours}{Ours\xspace}

\begin{abstract}
    In dense retrieval, deep encoders provide embeddings for both inputs and targets, and the softmax function is used to parameterize a distribution over a large number of candidate targets (e.g., textual passages for information retrieval). Significant challenges arise in training such encoders in the increasingly prevalent scenario of (1) a large number of targets, (2) a computationally expensive target encoder model, (3) cached target embeddings that are out-of-date due to ongoing training of target encoder parameters. This paper presents a simple and highly scalable response to these challenges by training a small parametric {\it corrector network} that adjusts stale cached target embeddings, enabling an accurate softmax approximation and thereby sampling of up-to-date high scoring ``hard negatives.’’ We theoretically investigate the generalization properties of our proposed target corrector, relating the complexity of the network, staleness of cached representations, and the amount of training data. We present experimental results on large benchmark dense retrieval datasets as well as on QA with retrieval augmented language models. Our approach matches state-of-the-art results even when no target embedding updates are made during training beyond an initial cache from the unsupervised pre-trained model, providing a 4-80x reduction in re-embedding computational cost.
\end{abstract}
\section{Introduction}

The softmax function, paired with deep neural encoder models, 
is often the parameterization of choice for discrete distributions
over many targets such as in classification \cite{logeswaran2019zero,yu2020pecos},
retrieval \cite{reddi2019stochastic,xiong2020approximate}, or reinforcement learning
\cite{dulac2015deep,gottipati2020learning}. 
This approach, often called a ``dual encoder,'' employs two separate deep networks, one to map an input to a fixed dimensional vector, another to map targets to the same vector space.  We then compute softmax logits as the inner product of an input vector to each target vector \cite{gillick2019learning,karpukhin2020dense,xiong2020approximate}.

With the typical softmax cross-entropy loss, exact training of the parameters of these two encoder networks would involve using the current parameters to compute the logits for all targets, requiring running the target encoder on all targets at every step of training.  Of course, this far-too-burdensome approach is not used in practice.  Instead, various approximations have been developed \cite{reddi2019stochastic,rawat2020doubly,lindgren2021efficient,xiong2020approximate,monath2023improving}.
The typical approximation computes a truncated softmax on a sampled subset of targets. These approaches store a \emph{cache} of ``stale'' encoded representations of targets and uses the stale, cached representations to draw samples from the softmax-parameterized distribution during training \cite{lindgren2021efficient,izacard2022few}. Previous work has used these stale representations amidst other approximations such as index structures \cite{xiong2020approximate,monath2023improving}, kernel-methods \cite{rawat2019sampled}, and focusing training on subsets of targets \cite{reddi2019stochastic}.  However, inevitably, the staleness of the target embeddings causes training regret.

In this work, we present a simple, general purpose method for addressing staleness in softmax-parameterized categorical distributions that is scalable enough to be updated at every step of training. Our approach improves upon an existing stale approximation using a learned \emph{target corrector network}. The target corrector network, inspired by recent work on training continuous energy-based models \cite{han2020joint,grathwohl2020learning,grathwohl2021mcmc}, is a small parametric model that accounts for the discrepancy between the stale approximation and unnormalized logits from the true distribution. By learning to improve upon the stale approximation, the target corrector network can be used to produce a more accurate approximation to the target distribution. We further extend beyond training large output space classifiers to latent variable retrieval augmented language models. 

In summary, the contributions of this paper are:

\textbf{Methodological} (\S \ref{sec:Method}) - We describe a novel training procedure for large output space models. It is based on approximating softmax-parameterized categorical distributions by using a parametric target corrector network that learns to improve stale approximations of logits.

\textbf{Theoretical} (\S \ref{sec:Analysis}) - We analyze the generalization properties of the corrector networks in terms of the discrepancy between the stale approximation and the true distribution, the complexity of the network, and the amount of training data. 

\textbf{Empirical} (\S \ref{sec:Experiments}) - 
      We evaluate our approach in training both dense retrieval models and latent variable retrieval augmented language models. Our approach matches the performance of much more computationally intensive approaches at a fraction of the computational expense.

\vspace{-3mm}
\section{Background}
\label{sec:Background}

\textbf{Softmax\hspace{1mm}} Given an input point $\point$, a distribution over a set of $\numtargets$ \emph{targets}, $\alltargets$, parameterized by the softmax function is:
\begin{equation}
\vspace{-1mm}
    P(\target|\point) = \frac{\exp(\temperature\logit_{\point,\target})}{\partitionfn_\point \triangleq \sum_{\target'\in \alltargets} \exp(\temperature\logit_{\point,\target'})},
    \label{eq:softmax}
\end{equation}
where $\temperature$ is the temperature. In this paper, we focus on applications in retrieval and latent variable models. For example, in Natural Questions \cite{kwiatkowski2019natural}, $\point$ refers to a question and targets, $\target$, correspond to Wikipedia passages.

\textbf{Dual-Encoders\hspace{1mm}} We compute the unnormalized logits, $\logit_{\point,\target}$, using a factorized representation. Deep parametric models, \emph{dual-encoders}, map the input, $\point$, and target, $\target$, to $\dimensions$-dimensional vectors, denoted $\queryencoder(\point; \parameters)$, and $\targetencoder(\target; \parameters)$:
\begin{equation}
    \vspace{-2mm}
    \logit_{\point,\target} = \langle \queryencoder(\point; \parameters), \targetencoder(\target; \parameters) \rangle.
    \label{eq:logits}
    \vspace{-2mm}
\end{equation}

\textbf{Training\hspace{1mm}} For a task-specific loss, $\mathcal{L}$, such as cross-entropy, dual-encoder parameters are optimized by gradient descent  \cite{rawat2019sampled}. However, exact computation of the normalizing constant, $\partitionfn_\point$, is typically intractable during training, since it would require computing $\targetencoder(\target)$ for millions or billions of targets. Instead of $P(\target|\point)$ in $\mathcal{L}$, a tractable (yet biased) approximation is to optimize the \emph{truncated} softmax, $\tilde{P}(\target|\point)$, including only a subset of targets $S(\alltargets) \subset \alltargets$:
\begin{equation}
    \tilde{P}(\target|\point) = \frac{\exp(\temperature\logit_{\point,\target})}{\sum_{\target'\in S(\alltargets)} \exp(\temperature\logit_{\point,\target'})},
    \label{eq:trunc_sm}
\end{equation}
\textbf{Uniform Sampling Approximation\hspace{1mm}} A simple approach is to define $S(\alltargets)$ to be a uniformly sampled subset of $\alltargets$  \cite{karpukhin2020dense}. 
The method's bias decreases with more samples. However, since the samples are uniform, a large number of samples may be required.

\textbf{Top-K / Similarity-based Sampling Approximations\hspace{1mm}} We can instead use an informed strategy using $\targetencoder(\target)$ that would select higher probability targets by sampling using similarity scores via Gumbel-Max \cite{lindgren2021efficient}, or using the top-k targets in terms of inner product \cite{xiong2020approximate}.
Work has considered efficient approximations to find these top $k$ targets without having to compute $\targetencoder(\target)$ for all $\target \in \alltargets$ \cite{xiong2020approximate,monath2023improving}.

\textbf{Initialization\hspace{1mm}} We initialize the parameters of the dual encoders, $\parameters$, using pre-trained models, such as pre-trained language models, T5 and GTR \cite{devlin2019bert,raffel2020exploring,ni2022large}.

\textbf{Stale Cached Representations\hspace{1mm}} 
When we are training the parameters, $\parameters$, each target's vector according to $\targetencoder(\target;\parameters)$ changes at each step of training. Therefore, a commonly used approach is to define an approximation, $\staletargetencoder(\target)$, that is a lookup for a ``stale''  cached embedding for the given target. The stale embedding comes from running the target encoder at a particular time step $t$, of training, and caching the result, i.e., $\targetencoder(\target,\parameters_t)$ in a buffer, $B \in \reals^{|\alltargets| \times D}$, i.e. $\staletargetencoder(\target_i) \triangleq \buffer_{\target_i}$.
To find the top-$k$ targets for input $\point$ we compute approximate logits $\buffer^T \queryencoder(\point) \in \reals^{|\alltargets|}$ and select the top-$k$ targets to define $S(\alltargets)$. Even before training, we can use the pre-trained model to produce embeddings for all targets $\buffer_{\target_i} = \targetencoder(\target_i; \parameters_0)$.  While $\buffer$ may seem large, this is considerably more efficient than exact computation and is possible, on accelerators, for $|\alltargets|$ in the tens of millions. 

The bias of this approach (and subsequent degradation in performance) depends on the \emph{staleness} or {drift} of the embeddings, i.e., $||\buffer_{\target_i} - \targetencoder(\target_i)||$ which will increase as we update the parameters of $\targetencoder(\target)$. This can be mitigated by recomputing $\buffer$ periodically throughout training (at notable cost). This approach of periodically recomputing has been used \cite{guu2020retrieval,izacard2022few,monath2023improving}, but there is still much room for improvement.

\vspace{-3mm}
\section{Improving Training with Target Correctors}
\label{sec:Method}
Our proposed approach builds upon these stale buffer approximations by using an additional parametric model. The additional model aims to improve upon the stale $\staletargetencoder(\target)$ to yield a better approximation of $\targetencoder(\target)$.

We refer to this additional parametric model as a \emph{target corrector network},  $\approxmodel(\cdot;\approxparameters)$ or simply $\approxmodel(\cdot)$ when the parameters $\approxparameters$ are not pertinent. This target corrector network takes as input the existing stale vector embedding, $\staletargetencoder(\target)$, and yields the following approximation of the softmax function:
\vspace{-1.5mm}
\begin{equation}
    P_\approxmodel(\target|\point) \propto \exp(\temperature\langle \queryencoder(\point), \approxmodel \circ \staletargetencoder(\target) \rangle).
    \label{eq:phyx}
    \vspace{-3mm}
\end{equation}

With significantly fewer parameters than a typical dual-encoder, i.e., $|\approxparameters| \ll |\parameters|$, this small parameteric model is efficient enough to provide approximately fresh representations of every target at every training step. The target corrector network presents interesting research questions regarding whether the network can obviate the need for re-embedding, what kinds of staleness or drift can be effectively modeled, and how much training data is required. 
\begin{figure}
    \centering
    \includegraphics[width=0.9\columnwidth]{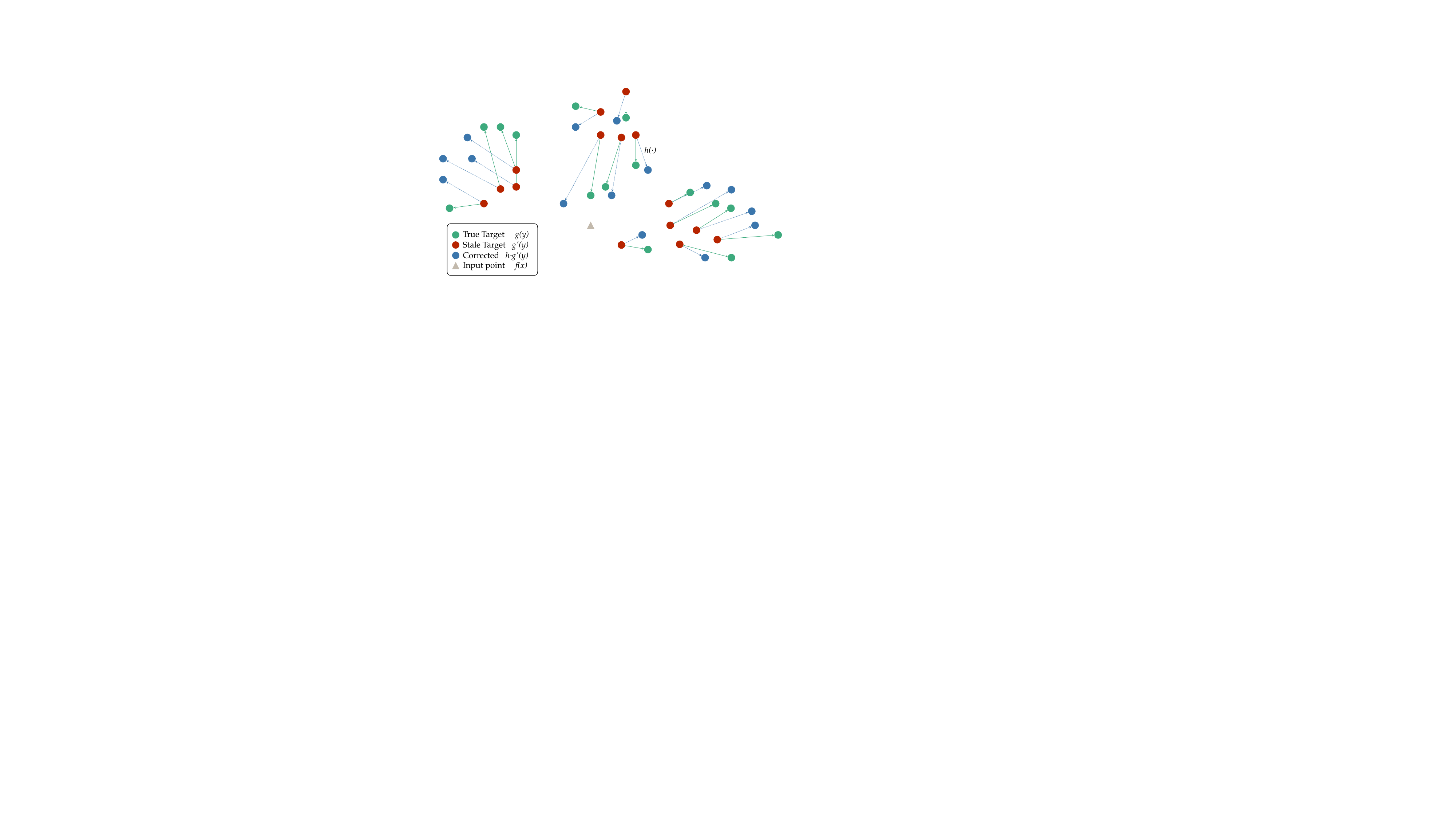}
    \caption{\textbf{Target Corrector Networks}. The corrector network, $\approxmodel(\cdot)$, moves the approximate representations of targets, $\staletargetencoder(\cdot)$ to be closer to their true positions, $\targetencoder(\cdot)$. The corrector network is trained to approximate how the targets are transformed from $\staletargetencoder(\cdot)$ to $\targetencoder(\cdot)$.}
    \label{fig:main}
    \vspace{-8mm}
\end{figure}

\textbf{Warmup: Training the corrector network in isolation\hspace{1mm}} 
We begin by considering how we would train only the parameters of the target corrector network, independently of the dual-encoders $\queryencoder(\point)$ and $\targetencoder(\target)$. Afterwards, we present an algorithm for jointly training the target corrector network and the dual-encoders.
To train the parameters $\approxparameters$ of the corrector network, $\approxmodel(\cdot; \approxparameters)$, we collect training examples of input data points $\point_i$, the exact target embeddings $\targetencoder(\target)$, and stale embeddings $\staletargetencoder(\target)$ for a subset of targets $S(\alltargets)_{i} \subset \alltargets$, i.e., $\{(\queryencoder(\point_i), \targetencoder(\target_b),\staletargetencoder(\target_b)) \ | \ \target_b  \in  S(\alltargets)_{i}\}$. 

We consider two loss functions for training $h$: the mean-squared error between representations given by $\targetencoder(\target)$ and the corrected representations $\approxmodel \circ \staletargetencoder(\target)$ (Eq.~\ref{eq:mse}) and the cross entropy loss between the truncated softmax using $\targetencoder(\target)$ and truncated softmax using $\approxmodel \circ \staletargetencoder(\target)$ (Eq.~\ref{eq:ce}):
\begin{align}
    \ell_{MSE}(y_i) &= || \targetencoder(y_i) - \approxmodel \circ \staletargetencoder(y_i;\approxparameters)  ||_2 \label{eq:mse} \\ 
    \ell_{\drm}(y_i) &= \log \tilde{P}(\target|\point) - \log \tilde{P}_\approxmodel(\target|\point) \label{eq:ce}\\
    \tilde{P}_{\approxmodel}(\target|\point) &= \frac{\exp(\temperature\langle \queryencoder(\point), \approxmodel \circ \staletargetencoder(\target; \approxparameters) \rangle)}{\sum_{ {\target' \in S(\alltargets)_i}} \exp(\temperature\langle \queryencoder(\point), \approxmodel \circ \staletargetencoder(\target';\approxparameters) \rangle)}. \nonumber
\end{align}
where $\tilde{P}(\target|\point)$ is the truncated softmax $g(y)$ (Eq.~\ref{eq:trunc_sm}). The mean-squared error loss directly tries to match the target encoder model's embeddings. 
The cross-entropy loss down-weights the importance of targets $\target$ which do not contribute substantial probability to $P(\target|\point)$ and allows for greater use of model capacity. 
The parameters of the target corrector networks are optimized using gradient descent.
Empirically, we find the cross-entropy objective to perform slightly better (Table~\ref{tab:nq_gtr_table_combo}) and focus the presentation on cross-entropy.

\textbf{Jointly Training Corrector Networks \& Dual-Encoders\hspace{1mm}}
We present a method (Algorithm~\ref{alg:training}) for simultaneously training dual-encoders for a given task (e.g., retrieval or equivalently large output-space classification)
and the target corrector network. The training algorithm will optimize both the 
parameters of the target corrector network and additionally use the corrector network to approximate the softmax. Each step consists of: (1) using the corrector network to provide an approximately updated representation of every target, (2) picking a subset of targets for the truncated softmax using the output of the  corrector network, (3) computing a task loss for the dual-encoder models and loss for the corrector networks, (4) updating, according to their respective losses, the parameters for both the dual-encoders and the corrector networks using gradient descent. 

In more detail, we are given task training data, $X = \{(x_1,y_1),\dots, (x_m,y_m)\}$. We are given a task loss function $\mathcal{L}$ and a corrector network loss $\ell$. The dual-encoder models are $\queryencoder(\point),\targetencoder(\target)$ and their initial parameters are $\parameters_0$. 
Prior to the first training step, we instantiate a buffer of the targets' representations, $\buffer_\target = \staletargetencoder(\target) = \targetencoder(\target;\parameters_0)$. We will avert the need for the expensive updating of the buffer by re-embedding targets with the target encoder.
In each step, we sample a training point and label pair $\point_i,\target_i$ from $X$. We apply the target corrector network to all of the stale representations in the buffer to obtain $\approxmodel \circ \staletargetencoder(\target) \ \forall \target \in \alltargets$. This computation does not require running a dual-encoder; we use the cached buffer representation of each target as input to the corrector network. The corrector network is typically a two-layer MLP and hence efficient enough to be used in this way. With these representations from $\approxmodel(\cdot)$, we sample (or select exact top-$k$) targets according to $P_\approxmodel(\target|\point)$ (Eq.~\ref{eq:phyx}) to form a subset of targets $S_{\point_i}(\alltargets)$ for the truncated softmax. 

Given this subset, we compute the task and correction losses and update their respective model parameters. First, we compute the task loss, which is cross-entropy. The task loss will only be used to update the parameters of the dual-encoders, $\parameters$, not the parameters of the target corrector network. We compute the truncated softmax $\tilde{P}(\target|\point) \propto \exp(\temperature\langle \queryencoder(\point), \targetencoder(\target) \rangle)$ (Equation~\ref{eq:trunc_sm}). We define a one-hot $P^\star$ according to the training data label $\target_i$. We compute the task specific loss $\mathcal{L}$ as a function of $\tilde{P}$ and $P^\star$, and update the dual encoder parameters via gradient descent $\parameters \gets \parameters - \eta \nabla_\Theta \mathcal{L}$. 

Next, we will use the same sample of targets $S_{\point_i}(\alltargets)$ to compute the target corrector network's loss and parameter update. Importantly, this will only update the parameters of the target corrector network, $\approxparameters$, not the parameters of the dual-encoders. Here we describe the use of the cross-entropy loss. However, an analogous update procedure could be used for other loss functions.  We compute the truncated softmax according to the target corrector network's output: $\tilde{P}_\approxmodel(\target|\point) \propto \exp(\temperature\langle \queryencoder(\point), \approxmodel \circ \staletargetencoder(\target) \rangle)$. We then compute the target corrector network loss, $\ell$, cross-entropy, which tries to align two truncated distributions $\tilde{P}_\approxmodel$ and $\tilde{P}$. The target corrector network's parameters are updated by gradient descent $\approxparameters \gets \approxparameters - \eta \nabla_\approxparameters \ell$. 

Training the target corrector network, which has only a small number of parameters, is much less computationally intensive to train than the dual-encoder model. Furthermore, we are given ``for free'' the representations $\targetencoder(\target)$ since they are used to compute $\tilde{P}$ for the task loss. These representations can then easily be re-used for training the corrector.

The training procedure is summarized in Algorithm
~\ref{alg:training}. At prediction time, the corrector network is not used, instead the trained dual-encoder $\targetencoder(\target, \parameters)$ is used.

\vspace{-2.2mm}
\SetKwComment{Comment}{/* }{ */}
\begin{algorithm}
\caption{Training with target corrector networks}\label{alg:training}
\KwData{Training data $X$, Targets $\alltargets$, Input encoder $\queryencoder(\cdot)$, Target encoder $\targetencoder(\cdot)$, Approximate target encoder $\staletargetencoder(\cdot)$ (buffer $B$), target corrector network $\approxmodel(\cdot)$, temperature $\temperature$, task loss $\mathcal{L}$, target corrector network loss $\ell$, learning rate $\eta$, number of truncated samples $k$}
\While{\emph{Training}}{
 \hspace{1mm} Sample training data $(\point_i, \target_i) \sim X$\\
 Compute $\approxmodel \circ \staletargetencoder(\target)$ for all $\target \in \alltargets$ using the buffer $B$ \\
  Set $S_{\point_i}(\alltargets)$ using $\exp(\temperature\queryencoder(\point_i)^T \approxmodel \circ \staletargetencoder(y))$ via top-$k$  \\
  Include supervised label $S_{\point_i}(\alltargets) \gets S_{\point_i}(\alltargets) \cup  \{y_i\}$ \\
 Define $\tilde{\P}(\target|\point_i) = \frac{\exp(\temperature\queryencoder(\point_i)^T \targetencoder(y))}{\sum_{y'\in S_{\point_i}(\alltargets)}\exp(\temperature\queryencoder(\point_i)^T \targetencoder(y'))}$\\
  Define ${\tilde{\P}_\approxmodel(\target|\point_i) = \frac{\exp(\temperature\queryencoder(\point_i)^T \approxmodel \circ \staletargetencoder(y))}{\sum_{y'\in S_{\point_i}(\alltargets)}\exp(\temperature\queryencoder(\point_i)^T \approxmodel \circ \staletargetencoder(y'))}}$\\
 Define $\P^\star$ to be a one-hot vector for $\target_i$. \\
 Compute task loss $\mathcal{L}$ using $\tilde{\P}$ and $\P^\star$ \\
 Compute correction loss $\mathcal{\ell}$ using $\tilde{\P}$ and $\tilde{\P}_\approxmodel$ \\
 Update dual-encoder parameters $\parameters \gets \parameters - \eta \nabla_\parameters \mathcal{L}$ \\
 Update corrector network parameters $\approxparameters \gets \approxparameters - \eta \nabla_\approxparameters \ell$ 
}
\end{algorithm}
\vspace{-3mm}

\subsection{Latent Variables in Retrieval Augmented Models}
\label{sec:RLM}

Retrieval augmented language models (RLMs) typically consist of two major architectural components, a \emph{retriever} model (e.g., a dual-encoder) and a generative language model or \emph{reader} model
\cite{guu2020retrieval,izacard2020leveraging,izacard2022few}. 
The input to a retrieval augmented language model is a natural language text sequence, $\point$. 
This input text will be encoded using a dual-encoder retrieval model, $\queryencoder(\point)$. Retrieval will be performed over a corpus of targets,   $\alltargets$, returning $k$ targets relevant to $x$, denoted $S_x(\alltargets)$. The reader model takes as input the retrieved targets, $S_x(\alltargets)$, and the text $x$, and generates text.

Concretely, in our experiments, the input text $\point$ is a question. The retrieval corpus contains targets $\target$ corresponding to passages in Wikipedia. The reader model takes as input the question and retrieved passages and generates a short answer to the question. We present the remainder of the section with this question-answering task in mind.

RLMs can be formalized as latent variable models. The softmax function is used to parameterize the distribution over a discrete latent variable, which corresponds to 
the retrieved targets. We use $a$ to refer to the generated sequence of text, i.e., the generated answer:
\begin{equation}
\vspace{-1mm}
    P(a | \point) = \sum_{\target \in S_\point(\alltargets)} P(a | \target, \point) P (\target | \point).
    \vspace{-1mm}
\end{equation}
$P(a|\target, \point)$ is an autoregressive language model. $P(\target|\point)$ is computed by the softmax with logits from Equation~\ref{eq:logits} using the encoder models $\queryencoder(\point)$ and $\targetencoder(\target)$.

When training RLMs, we receive supervision in the form of question, answer pairs, e.g., $\point_i,a_i \sim X$. We do not receive supervision on {which} targets $S_\point(\alltargets)$ should be retrieved. We will learn the parameters of both the reader model and retriever model using these supervised question/answer pairs.

To train the reader and retriever model, we use perplexity distillation \cite{izacard2022few} for retriever loss and negative log-likelihood for the reader loss. Perplexity distillation is computed as the cross-entropy between two truncated distributions, one   being the retriever's $\tilde{P}(\target|\point)$ (Equation~\ref{eq:trunc_sm}) and the other using the reader model to provide a soft-relevance label for each target in $S_\point(\alltargets)$:
\begin{align}
    P_a(\target|\point) = \frac{P(a | \target, \point)}{\sum_{y'\in S_\point(\alltargets)} P(a|y', \point)}.
\end{align}
In words, $P_a(\target|\point)$ normalizes the likelihood scores of the reader model generating the correct answer text when conditioned on the given retrieved target $\target$. The reader's loss function, negative-log likelihood is simply computed using the supervised answer text. The two losses are averaged and parameters optimized with gradient descent.

To facilitate efficient training, we use our proposed target corrector network to select the subset of retrieved targets $S_\point(\alltargets)$ used at training time. This is done in the same way as in Algorithm~\ref{alg:training}, i.e.,  we pick a subset of $k$ targets $S_{\point}(\alltargets)$ for $\point$ according to $\exp(\temperature\queryencoder(\point)^T \approxmodel \circ \staletargetencoder(y))$ via top-$k$ or Gumbel-Max sampling. We can make simple modifications to Algorithm~\ref{alg:training}, which are presented in Algorithm~\ref{alg:rlm_training} to train the RLM. We compute two task-specific losses (perplexity distillation, negative log-likelihood) and optimize both the reader and retriever parameters. We use cross-entropy to train the corrector, which is again only used at training time. At prediction time, the trained retriever model is used.

\begin{algorithm}
\caption{Training RLMs with corrector models}\label{alg:rlm_training}
\KwData{Training data $X$, Targets $\alltargets$, Retriever and Reader Parameters $\parameters$, Correction Model parameters $\Psi$, Input encoder $\queryencoder(\cdot)$, Target encoder $\targetencoder(\cdot)$, Approximate target encoder $\staletargetencoder(\cdot)$ (buffer $B$), corrector model $\approxmodel(\cdot)$, temperature $\temperature$, retriever loss $\mathcal{L}$, reader loss $\mathcal{L}'$, corrector model loss $\ell$, learning rate $\eta$, number of truncated samples $k$}
\While{\emph{Training}}{
 \hspace{1mm} Sample training data $(\point_i, a) \sim X$\\
 Compute $\approxmodel \circ \staletargetencoder(\target)$ for all $\target \in \alltargets$ using the buffer $B$ \\
  Set $S_{\point_i}(\alltargets)$ using $\exp(\temperature\queryencoder(\point_i)^T \approxmodel \circ \staletargetencoder(y))$ via top-$k$  \\
  Define $\tilde{\P}(\target|\point_i) = \frac{\exp(\temperature\queryencoder(\point_i)^T \targetencoder(y))}{\sum_{y'\in S_{\point_i}(\alltargets)}\exp(\temperature\queryencoder(\point_i)^T \targetencoder(y'))}$\\
  Define ${\tilde{\P}_\approxmodel(\target|\point_i) = \frac{\exp(\temperature\queryencoder(\point_i)^T \approxmodel \circ \staletargetencoder(y))}{\sum_{y'\in S_{\point_i}(\alltargets)}\exp(\temperature\queryencoder(\point_i)^T \approxmodel \circ \staletargetencoder(y'))}}$\\
 Define $P_a(\target|\point) = \frac{P(a | \target, \point)}{\sum_{y'\in S_\point(\alltargets)} P(a|y', \point)}.$ \\
 Define $P_\text{LM}(a | \point) = \sum_{\target \in S_\point(\alltargets)} P(a | \target, \point) P (\target | \point).$ \\
 Compute reader loss $\mathcal{L}'$ using $P_\text{LM}(a | \point)$ \\
 Compute retriever loss $\mathcal{L}$ using $\tilde{\P}(\target|\point_i)$ and $P_a(\target|\point)$ \\
 Compute correction loss $\mathcal{\ell}$ using $\tilde{\P}$ and $\tilde{\P}_\approxmodel$ \\
 Update retriever \& reader params $\parameters \gets \parameters - \eta \nabla_\parameters \frac{\mathcal{L} + \mathcal{L}'}{2}$ \\
 Update corrector network params $\approxparameters \gets \approxparameters - \eta \nabla_\approxparameters \ell$ 
}
\end{algorithm}

\vspace{-2mm}
\section{Analysis}
\label{sec:Analysis}

We will explore the generalization of the proposed target corrector network in terms of unseen targets for a particular input data point, and will show the relationship between generalization error, the complexity of the target corrector network $\approxmodel$, and the discrepancy of the stale representations, $\staletargetencoder$, and true representations $\targetencoder$. All proofs are in Appendix~\ref{appendix:proofs}.

Let $\ell : \mathbb{R} \times \mathbb{R} \to \mathbb{R}$ is a loss function for the target corrector network (Eq.~\ref{eq:mse} \& \ref{eq:ce}). 
 
For any point $\point$, consider the distribution given by the softmax using stale approximation $\staletargetencoder$:
\vspace{-2mm}
\begin{equation}
    \tilde{\DCal}_Y \triangleq {P}_{\staletargetencoder}(\target|\point) = \frac{\exp(\temperature  \langle \queryencoder(\point), \staletargetencoder(\target) \rangle)}{\sum_{\target'\in \alltargets} \exp(\temperature  \langle \queryencoder(\point), \staletargetencoder(\target') \rangle)},
\end{equation}
and similarly define ${\DCal}_Y \triangleq {P}_{\targetencoder}(\target|\point)$ as the true distribution, using $\targetencoder$ (Eq.~\ref{eq:softmax}). 

We begin by defining three kinds of risk. 

\textbf{Empirical Risk\hspace{1mm}} On a set of $n$-targets $\tilde{\SCal}_n=\{y_1,...,y_n\}$ sampled from $\tilde{\DCal}_Y$, we minimize the empirical risk: 
\vspace{-1mm}
\begin{equation}
    R_{\ell,{\phi}}(\tilde{\SCal}_n) = \frac{1}{n} \sum_{i=1}^n \ell(\phi(y_i), g(y_i)),
\end{equation}
\vspace{-2mm}
over a function class $\phi\in\FC$.

\textbf{True Population Risk\hspace{1mm}} For generalization error, we are interested in how large the true population risk can become over a function class $\phi\in\FC$.
    \vspace{-1mm}
\begin{equation}
    R_{\ell, \phi}(\DCal_Y) = \mathbb{E}_{Y\sim\DCal_Y}[\ell(\phi(Y), g(Y))],
\end{equation}
We consider the above quantity because we want to ensure good alignment between $g(y)$ and $\phi(y)$ where there is non-trivial probability mass under the true distribution.

\textbf{Stale Population Risk\hspace{1mm}} The stale population risk is defined analogously to true population risk with $\tilde{\DCal_Y}$ as the distribution, over a function class $\phi\in\FC$:
\vspace{-1mm}
\begin{equation}
     R_{\ell, \phi}(\tilde{\DCal}_Y) = \mathbb{E}_{Y\sim\tilde{\DCal}_Y}[\ell(\phi(Y), g(Y))].
     \vspace{-3mm}
\end{equation}

\textbf{Function Classes\hspace{1mm}} The function class $\phi\in\FC$ is large. We will relate this large function class to a restricted class of functions of the form $h \circ g'$ by leveraging the approximate stale representations, $\staletargetencoder$. 
In other words, we restrict $\FC$ to $\FC^{g'} = \{h \circ g' : h \in \HC\}$ where $\HC$ represents the simpler function class mapping $\mathbb{R}^d\to\mathbb{R}^d$ which can express the discrepancy between the stale $g'$ and current $g$.

First, we provide a bound on the gap between the population risk and stale population risk. We formalize this in the following lemma. For ease of notation in this exposition, we define $\GC_{\ell,  \FC}$ as the induced function class:
$\GC_{\ell,  \FC} = \{y \mapsto  \ell(\phi(y), g(y)) : \phi \in \FC\}.$

\begin{restatable}{lemma}{gapTprSpr}\label{lemma:gap}
Given a target encoder $\targetencoder$ and its stale approximation $\staletargetencoder$, the gap between the true population risk and stale population risk is bounded in the following way:
\begin{equation}
    R_{\ell, \phi}(\DCal_Y) - R_{\ell, \phi}(\tilde{\DCal}_Y) \leq \WC(\DCal_Y, \tilde{\DCal}_Y) \leq  \|g - g'\|_1
\end{equation}
where $\WC$ is the Wasserstein distance. Furthermore, if the approximation $\staletargetencoder$ comes from the same neural model as  $\targetencoder$ with parameters perturbed by $u$ as in aforementioned stale approximation, we have: $\|g-g'\|_1 \leq L\|u\|$ with $L$ as the Lipschitz constant. 
\end{restatable}

Next, we connect stale population risk to the empirical risk.
\begin{restatable}{lemma}{staleToEmpirical}\label{lemma:staleToEmpirical}
Given a target encoder $\targetencoder$, its stale approximation $\staletargetencoder$, and a set of $n$-targets $\tilde{\SCal}_n=\{y_1,...,y_n\}$ sampled from $\tilde{\DCal}_Y$,
\begin{equation}
R_{\ell, \tilde{\phi}_n}(\tilde{\DCal}_Y) \leq R_{\ell,\tilde{\phi}_n}(\tilde{\SCal}_n) + {\RF}_{\tilde{\SCal}_n}(\GC_{\ell,  \FC}),
\end{equation}
where ${\RF}_{\tilde{\SCal}_n}(\GC_{\ell,  \FC})$ is the Rademacher complexity of $\GC_{\ell,  \FC}$.
\end{restatable}

Now, we can relate the complexity of function class $\FC^{g'}$, number of samples $n$, and the discrepancy of the true $\targetencoder$ and stale approximate encoders $\staletargetencoder$:
\begin{restatable}{theorem}{maintheorem}\label{thm:maintheorem}
For a target encoder, $\targetencoder$, its stale approximation, $\staletargetencoder$, and the Rademacher complexity  $\tilde{\RF}_{n}(\GC_{\ell,  \FC^{g'}})$, the true population risk $R_{\ell, \phi}(\DCal_Y)$ is bounded by the following   with probability at least $1-\delta$:
\begin{align}
    R_{\ell, \phi}(\DCal_Y)& \leq T_1 + T_2 + T_3  \\
    T_1 &= R_{\ell,\tilde{\phi}_n}(\tilde{\SCal}_n) \nonumber \\
    T_2 &= {\WC(\DCal_Y, \tilde{\DCal}_Y)}\ {\leq L \|u\|} \nonumber \\
    T_3 &= \tilde{\RF}_{n}(\GC_{\ell,  \FC^{g'}}) + \OC\Big(\sqrt{{\frac{\log(1/\delta)}{n}}}\Big) \nonumber 
\end{align}
\end{restatable}

Note the following implications of these theoretical results:

   1. If the corrector network $\approxmodel$ is too complicated or there are not enough samples $n$, then $\approxmodel$ overfits and $T_3$,  will dominate.
    
    2. If $\targetencoder$ and $\staletargetencoder$ are very different, then term $T_2$ will dominate.
    
    3. If $\approxmodel(\cdot)$ is too simple and we cannot fit the sampled data well, then $T_1$ will dominate.
    
We empirically explore some of these trade-offs in \S\ref{sec:synth}.
\vspace{-3mm}

\begin{figure}
    \centering
    \includegraphics[width=0.8\columnwidth]{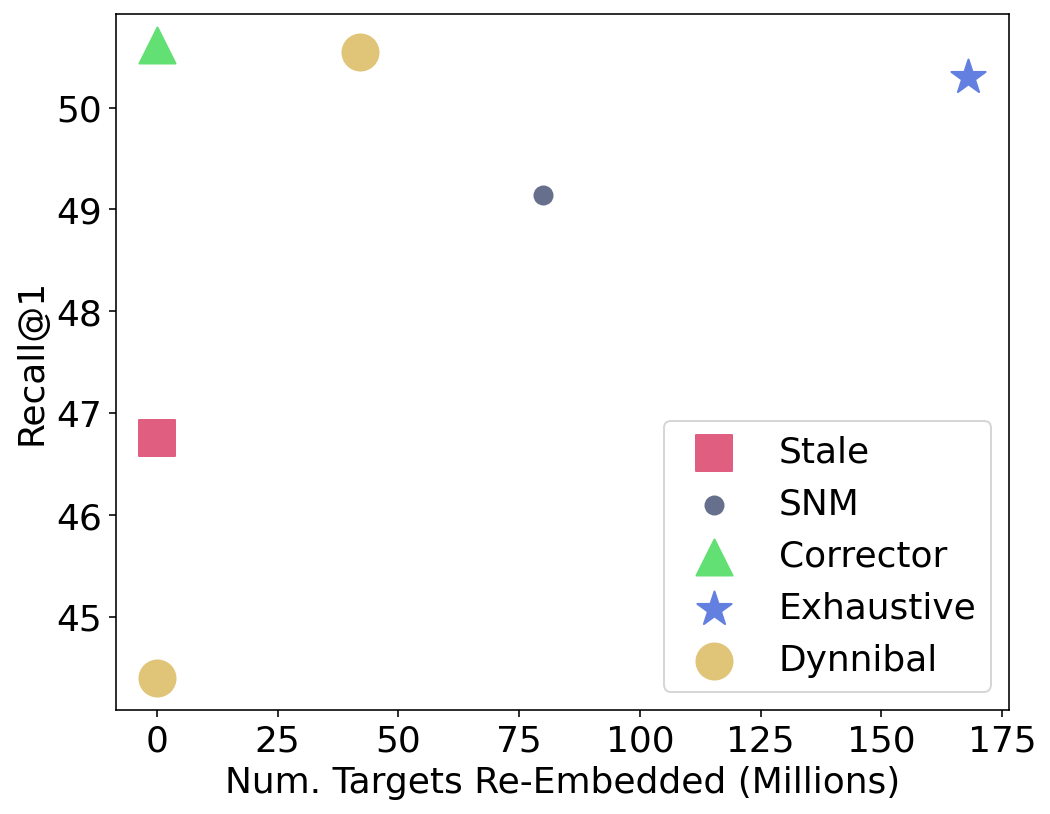}
    \vspace{-3mm}
    \caption{\textbf{NQ Test Recall@1}. We show the computational trade-offs between the amount of re-embedding during training and the task performance (GTR initialization). Our proposed target corrector approach achieves matching task performance at a fraction of the computational expense.}
    \label{fig:tradeoff}
    \vspace{-3mm}
\end{figure}

\section{Experiments}
\label{sec:Experiments}

We evaluate training using target corrector networks in two settings: supervised dense retrieval and retrieval augmented language models. 
We further investigate the properties of the target corrector networks in synthetic experiments.
In summary, the experiments investigate whether training strategies can effectively obviate the need to keep cached buffers of targets up-to-date by re-embedding during training. We answer this question affirmatively with the following highlights:

\textbf{No Re-embedding Needed} Training using target corrector networks matches the task performance of exhaustively re-embed all targets every 500 steps throughout training in both dense retrieval (Table~\ref{tab:nq_gtr_table_combo}) and retrieval augmented language models (Table~\ref{tab:rlm}). Target correctors achieve this without ever needing to re-embed targets during training, yielding significant computational savings (Fig.~\ref{fig:tradeoff}).

\textbf{Best no re-embedding method} Compared to frozen approaches, stale approaches, and Dynnibal without re-embedding, target corrector networks achieve over 10 point improvements in RLM tasks and 4 point improvements across multiple recall measures in retrieval.

\textbf{Simpler and Less Computation} Target correctors perform as well or better than Stochastic Negative Mining (SNM) \cite{reddi2019stochastic} despite SNM doing more re-embedding. Similarly, target corrector networks nearly match Dynnibal \cite{monath2023improving} when Dynnibal uses much more computation (Table~\ref{tab:nq_gtr_table_combo}).  Dynnibal is a much more complicated and difficult to implement method.

\subsection{Supervised Dense Retrieval}

\textbf{Setting \& Metrics\hspace{1mm}} We evaluate training methods for supervised dense retrieval models. Each method is provided the same supervised data. All methods use a stale buffer of target representations and use this buffer to form the subset of targets, $S(\alltargets)$, used in computing the truncated softmax. All methods use the same loss (cross-entropy) and optimize parameters of the dual-encoders using gradient descent. The methods differ in their maintenance of the buffer, and, as such, differ in their computational requirements of maintaining this buffer. We measure the computational requirements in terms of how many targets are re-embedding during training\footnote{Our JAX \cite{jax2018github} implementation run on Cloud TPUv3 re-embeds  \textasciitilde 2184 targets per second on each core.\vspace{-3mm}}. We measure re-embedding in terms of the number of targets encoded to indicate the computational expense (even if wall clock time is mitigated using a complicated asynchronous computation). Re-embedding every target even one additional time during training can be problematic if number of targets is large. Furthermore, the initial buffer, created using the initial parameters of the dual-encoder (e.g., a pre-trained language model) can be computed once and used for subsequent training jobs.

\textbf{Data\hspace{1mm}} We evaluate on Natural Questions \cite{kwiatkowski2019natural} with over 21M targets (Wikipedia passages), about 60K training examples (question, passage pairs), and about 3K in dev/test, and MSMARCO \cite{bajaj2016ms} 8.8M targets (web passages), and 500K training examples.

\textbf{Models\hspace{1mm}} We initialize the dual encoder models with  two publicly available pre-trained language models, GTR \cite{ni2022large}, and T5 \cite{raffel2020exploring}. GTR is an encoder model initialized from T5 and further pre-trained for dense retrieval on a large collection of corpora of question/answer pairs. For MSMARCO, we only use T5 since it is included in GTR's training data. We use the base size models, $D=768$, and train separate parameters for $\queryencoder(\point)$ and $\targetencoder(y)$. For the target corrector, we use a two layer MLP. We use 8192 hidden units, a ReLU non-linearity, and a residual connection. 

\begin{table*}[]
    \centering
    \begin{tabular}{ll@{}r r r  r r r | r r r r r@{}}
    \toprule
         &&  \multicolumn{1}{r}{Re-embed} & \multicolumn{5}{c|}{NQ Dev - Recall ($\uparrow$)} & \multicolumn{5}{c}{NQ Test - Recall ($\uparrow$)}  \\
         && Num. ($\downarrow$)  & @1 & @5 & @10 & @20 & @100 &  @1 & @5 & @10 & @20 & @100  \\
         \midrule
        \parbox[t]{2mm}{\multirow{7}{*}{\rotatebox[origin=c]{90}{GTR-base}}} 
        & In-batch & 0 & 17.14 & 46.77 & 58.71 & 69.45 & 85.54 & 37.92 & 64.76 & 72.54 & 78.28 & 87.00 \\
        & Stale & 0 & 33.11 & 62.04 & 70.31 & 78.13 & 89.32 & 46.76 & 68.64 & 75.21 & 80.66 & 87.48 \\
        & Dynnibal$^+$  & 0 & 28.73 & 59.66 & 70.08 & 78.14 & 90.18 & 44.40 & 67.53 & 74.93 & 80.22 & 87.23 \\
        & Corrector $^\text{\faNewspaper[regular]}$ {\scriptsize($\ell_\text{mse}$)} & 0 & 34.98 & 65.03 & 74.01 & 80.77 & 90.82 & 49.61 & 70.72 & 77.04 & 82.33 & 88.28 \\
        & Corrector $^\text{\faNewspaper[regular]}$ & 0 & \bf 35.78 & \bf 66.74 & \bf 75.06 & \bf 81.52 & \bf 91.37 & \bf  50.61 & \bf 71.00 & \bf 77.73 & \bf 82.66 & \bf  88.39 \\
        \cmidrule{2-13}
        & Dynnibal$^+$ & 42M & 35.86 & 66.54 & 75.04 & 81.40 & 91.27 & \underline{50.55} & \underline{71.69} & \underline{78.25} & \underline{83.35} & \underline{88.73} \\
        & SNM$^\dagger$ & 80M & 32.03 & 64.01 & 73.72 & 81.37 & 91.47  & 49.14 & 69.89 & 
        77.12 & 82.19 & 87.95 \\
        & Exhaustive & 1.68B & \underline{36.29} & \underline{67.08} & \underline{75.55} & \underline{82.07} & \underline{91.73} & 50.30 & 71.55 & 78.12 & 82.83 & 88.59 \\
         \midrule
        \parbox[t]{2mm}{\multirow{7}{*}{\rotatebox[origin=c]{90}{T5-base}}} 
        & In-batch & 0 & 9.93 & 28.07 & 37.17 & 45.54 & 64.06  & 23.40 & 47.50  & 56.39 & 65.34  & 77.97  \\
        & Stale & 0 & 16.79 & 36.85 & 44.82 & 51.79 & 67.35 &  27.65 & 50.19 & 59.28 & 66.98 & 78.95 \\
        & Dynnibal $^+$ & 0 & 17.42 & 39.65 & 48.75 & 57.36 & 73.03 & 29.72 & 53.99 & 63.38 & 70.61 & 80.94 \\
        & Corrector $^\text{\faNewspaper[regular]}$ & 0 & \bf 23.64 & \bf 47.69 & \bf 56.68 & \bf 64.65 & \bf 79.03 & \bf 36.65 & \bf 59.25 & \bf 68.06 & \bf 73.71 & \bf 83.13 \\
        \cmidrule{2-13}
        & Dynnibal $^+$ & 42M & 23.71 & 46.63 & 55.75 & 63.88 & 79.46 & 36.65 & 59.31 & 67.65 & 74.46 & 83.13 \\
        & Dynnibal $^+$ & 80M & \underline{24.76} & 47.69 & 56.82 & 64.90 & \underline{80.15} & 36.90 & 59.97 & 68.23 & 74.54 & 83.35 \\
        & SNM $^\dagger$  & 80M &  22.55 & 46.86 & 55.72 & 64.19 & 80.40  & 35.93 & 59.06 & 67.48 & 73.66 & 82.85 \\
        & Exhaustive & 1.68B & 24.70 & \underline{48.21} & \underline{57.18} & \underline{65.39} & 79.94 & \underline{37.34} & \underline{60.42} & \underline{68.70} & \underline{74.76} & \underline{83.41} \\
        \bottomrule
    \end{tabular}
    \vspace{-3mm}
        \caption{\textbf{Natural Questions} \cite{kwiatkowski2019natural}. $\text{\faNewspaper[regular]}$ This paper; $\dagger$ \cite{reddi2019stochastic}' $+$ \cite{monath2023improving}. We measure the performance of dense retrieval models trained with different softmax approximations via hard negatives. We find that our proposed approach nearly matches the performance of exhaustively re-embedding all targets. Similarly, our approach requires significantly less re-embedding of targets than Dynnibal, which requires at least one exhaustive re-embedding to get competitive results. \textbf{Bold-face numbers} indicate best performance with zero re-embeddings performed at training time; \underline{underlined numbers} indicate best performance using re-embedding at training time.}
    \label{tab:nq_gtr_table_combo}
    \vspace{-5mm}
\end{table*}

 We compare the following approaches: \textbf{Target Corrector Networks} (this paper): At the first training step, we initialize the buffer with vector representations of every target. At every subsequent step, we use the target corrector network to produce a new representation of the targets, without running the target-encoder, simply by running our small MLP corrector on the stale representations. The stale buffer representations are never updated during training. \textbf{Stale}: We initialize the buffer of targets at the first step of training and do not update it throughout training. We experimented with both freezing the target encoder parameters $\targetencoder(\target)$ and allowing them to be updated despite the stale buffer. We found updating the parameters to be slightly better and report those results. \textbf{Exhaustive}: We exhaustively re-embed all of the targets in the buffer every 500 steps of training. \textbf{Stochastic Negative Mining} (SNM; \citealt{reddi2019stochastic}): Instead of storing every target in the buffer, we store a subset of targets sampled uniformly at random. We re-sample and re-embed this buffer every 500 steps. We use a buffer size of 1M targets. \textbf{Dynnibal} \cite{monath2023improving}: This complicated approach maintains a buffer using a low-rank regression model as a part of tree index structure. The regression model is updated every 500 steps on a sub-sample of targets, unlike our approach which is trained jointly. Furthermore, to get good performance, Dynnibal performs costly full buffer re-embedding periodically throughout training. We needed to perform two such re-embeddings. Dynnibal with fewer refreshes does not perform as well.

\begin{table}
\vspace{2mm}
    \centering
    \begin{tabular}{@{}l @{} r  r  r  r@{}  r@{}}
    \toprule
          & Hidden  & Steps/sec & R@1 & R@20 \ \  & R@100 \\
          & Units \\
         \midrule
        Exhaustive & - & 0.43&	36.29&	82.07 \ \ &	91.70 \\
        Corrector & 8192 & 	0.83 &	35.78&	81.52 \ \ &	91.37   \\
        Corrector & 4096 & 	1.10&	35.53&	81.63 \ \ &	91.07   \\
        Corrector & 2048 & 1.83&	35.55&	81.07 \ \ &91.08 \\
    \bottomrule
    \end{tabular}
    \vspace{-2mm}
    \caption{\textbf{Steps-Per-Second Comparison.} We compare on NQ-dev the performance and speed of corrector networks and exhaustive re-encoding. Models are finetuned GTR-base.}
    \label{tab:timecomp}
    \vspace{-5mm}
\end{table}

In Table~\ref{tab:nq_gtr_table_combo}, our  
target corrector network approach greatly improves upon the stale approach, especially in Recall@1, 5, 10. We observe a nearly 5 point improvement at R@10 in the dev set and a 4 point improvement in R@1 on the test over the stale approach. Our approach nearly matches the performance of the computationally intensive exhaustive approach. Furthermore, we perform comparably to the more expensive SNM and Dynnibal methods. We perform better than Dynnibal for the same amount of re-embedding. While doubling the number of index refreshes may appear negligible, having to re-embed the buffer during training can be computationally burdensome, especially as the number of targets grows. Using a buffer created from the initial parameters of the dual-encoder as with our approach, allows the buffer to be constructed once ahead of time and re-used across both training and tasks. Dynnibal requires hand tuning to get the re-embedding schedule correct. 

Table~\ref{tab:nq_gtr_table_combo} also compares dual-encoder initialization. GTR is pre-trained for retrieval and hence achieves better results. T5 is not pre-trained for retrieval and requires more adaptation for the retrieval task. We observe that SNM struggles more to match the performance of Exhaustive with T5. Furthermore, Dynnibal requires more full index refreshes to get competitive results. Our method is able to achieve nearly as good results as the Exhaustive approach and Dynnibal (with re-embedding) despite never needing to re-embed. 

We also report timing comparisons in terms of steps-per-second between corrector networks (of two sizes) and exhaustive re-encoding of the targets. These can be found in table \ref{tab:timecomp}. We can see that both small and large corrector networks lead large speed gains over exhaustive re-encoding with minimal performance gains. This indicates that corrector networks can have practical training time efficiency gains over exhaustive methods.

See Appendix~\ref{appendix:additionalresults} for additional results (MSMARCO, other ablations) and further discussion.

\begin{table}
    \centering
    \begin{tabular}{@{}l @{}r r r r r@{}}
    \toprule
         &Re-embed & Retr. & NQ & TQA & HPQA  \\
         & Num. ($\downarrow$) \\
         \midrule
         No Retr.  &0& - & 25.4 & 26.1 & 14.5 \\
         \midrule
        Frozen Retr. &0 & GTR  & 48.4 & 55.1 & 28.0 \\
        Corrector & 0 & GTR   & 52.3 & 66.4 & 36.7 \\
        Exhaustive & 1.1B & GTR  & 52.4 & 66.5 & 33.8 \\
        \midrule
        Frozen Retr. & 0 & T5  & 13.34  & 12.15 & 13.37 \\
        Corrector& 0 & T5   &  48.1  & 63.73 & 21.97 \\
        Exhaustive & 1.1B & T5  & 48.3 & 66.03 & 25.45 \\
    \bottomrule
    \end{tabular}
    \vspace{-2mm}
    \caption{\textbf{Exact Match Accuracy, RLM Training}. We find that our target corrector approach can match the performance of the fully refreshed index while never re-embedding the targets across NQ, TriviaQA (TQA), and HotPotQA (HPQA).}
    \label{tab:rlm}
    \vspace{-5mm}
\end{table}

\subsection{Retrieval Augmented Language Models}

\textbf{Setting \& Metrics\hspace{1mm}} We evaluate the latent variable use case of training the retriever in a retrieval-augmented language model (RLM), as described in Section~\ref{sec:RLM}. We will compare approaches for training in terms of their re-embedding costs. 

\textbf{Datasets\hspace{1mm}}  We evaluate on the three question answering datasets: TriviaQA \cite{joshi2017triviaqa}, NQOpen \cite{kwiatkowski2019natural}, and HotPotQA \cite{yang2018hotpotqa}. We use 256 token passages from a 2018 Wikipedia snapshot as the collection of targets, $\alltargets$, with 28M targets.

\textbf{Models\hspace{1mm}} We initialize the retriever with GTR-base or T5-base and use T5-base as the reader  in Fusion-In-Decoder \cite{izacard2020leveraging}. We use 32 retrieved documents in all experiments. The target corrector is a two-layer MLP.

We compare the following approaches: \textbf{Target Corrector Network}: Target corrector is used to retrieve $S(\alltargets)$ at training time. We embed the targets at the beginning of training and never update the buffer. \textbf{No Retrieval} \cite{roberts2020much}: The retriever is not used. The reader model is trained on the dataset and uses only its parameters to answer the questions. \textbf{Frozen Retrieval} \cite{izacard2020leveraging}: Every target is embedded once at the beginning of training. Only the parameters of the reader model are trained (updating the retriever parameters did not improve performance). \textbf{Exhaustive}: Jointly training the retriever and reader, we exhaustively re-embed all 28M targets every 500 steps.

In Table~\ref{tab:rlm}, we report exact match accuracy on the held-out validation sets. Our proposed target corrector matches or nearly matches the performance of the exhaustive re-embedding approach without ever having to re-embed the buffer. This is a dramatic reduction in computational cost, as the exhaustive approach ends up embedding all 28M passages 40 times (1.1B re-embeddings). Target correctors greatly outperform the approaches that do not use retrieval (by more than 20 points) and the frozen retriever approach (by at least 4 points and by up to 10 points).

\begin{figure}
    \centering
    \includegraphics[width=0.95\linewidth]{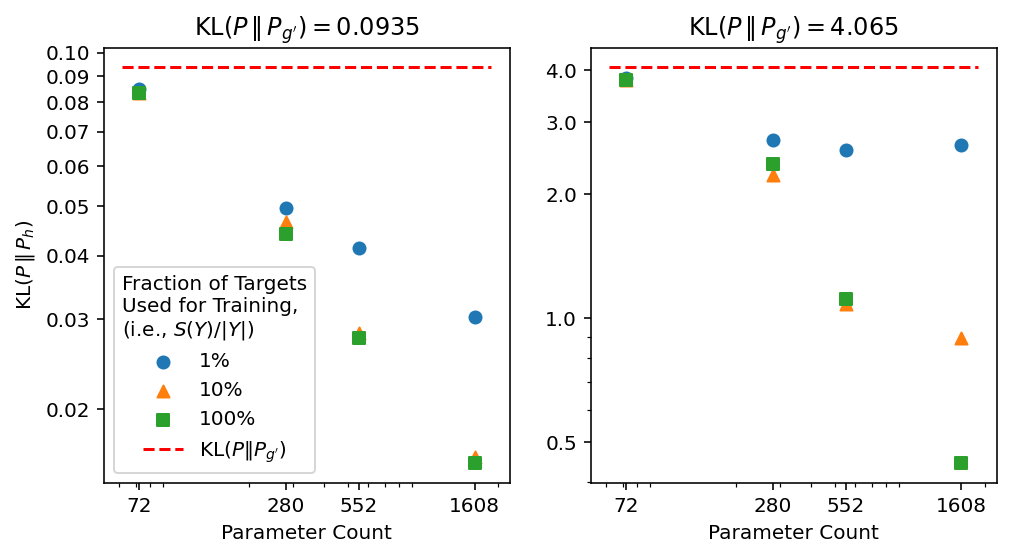}
    \vspace{-3mm}
    \caption{\textbf{Training Sample Size} The figure shows the trade-offs between the complexity of $h$ (parameter count), the approximation error, $KL(P||P_\approxmodel)$, and the fraction of samples used for training. The left hand side shows only somewhat stale representations. The right hand side shows significantly stale representations. Using a higher fraction of training samples is needed with more staleness.}
    \label{fig:sample_size}
    \vspace{-6mm}
\end{figure}

\subsection{Synthetic Experiments}
\label{sec:synth}

In these experiments, we measure the ability of proposed
corrector network to approximate categorical distributions parameterized by the softmax by training the corrector network, $\approxmodel$, without training parameters of the dual-encoder.

\textbf{Setting \& Metrics} We will measure the ability of proposed
corrector network to approximate categorical distributions parameterized by the softmax. We do so by training the corrector network, $\approxmodel$, in isolation, e.g., only training the parameters of the corrector network, $\approxparameters$, without training parameters of the dual-encoder for a particular task. We measure the quality of approximation using the KL-divergence between the true categorical distribution $P(\target|\point)$ (Equation~\ref{eq:softmax}) and the approximate distribution given by the corrector network $P_\approxmodel(\target|\point)$ (Equation~\ref{eq:phyx}). We measure the complexity of the corrector network by its parameter count, $|\approxparameters|$. We measure staleness, i.e., the difficulty of correcting a set of stale representations, by the KL-divergence between the true categorical distribution $P(\target|\point)$ and the distribution $P_{\staletargetencoder}(\target|\point) \propto \exp(\temperature\langle \queryencoder(\point), \staletargetencoder(\target) \rangle)$.

\begin{figure}
    \centering
    \includegraphics[width=0.8\linewidth]{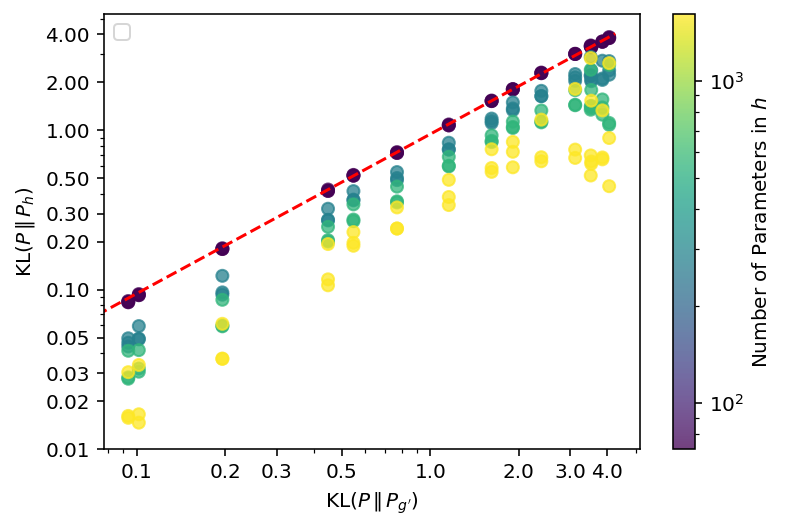}
    \vspace{-4mm}
    \caption{\textbf{Parameter Count} We plot the
    KL divergence using the stale embeddings ($\operatorname{KL}(P \Vert P_{g'})$, on the x-axis) against that of the trained correction models ($\operatorname{KL}(P \Vert P_\approxmodel)$, on the y-axis). Parity is indicated by the dashed red line, demonstrating that the trained correction model is significantly better than using the stale embeddings. Increasing the parameter count of $h$ is more important when the discrepancy between stale and current embeddings is higher, indicated by a larger improvement toward the right of this plot.}
    \label{fig:param_count}
    \vspace{-12mm}
\end{figure}

\textbf{Data Generation} We directly generate vector representations corresponding to data points and targets. That is, rather than having a dual-encoder model provide the vector representation of a data point or target, we directly generate synthetic data corresponding to $f(x)$, $g(y)$, and $g'(y)$. We generate 4096 targets in $D=8$ dimensions from a mixture of 20 Gaussians to represent $g'(y)$. To generate $g(y)$, we transform $g'(y)$ by feeding the points into randomly initialized  MLPs with up to $2$ hidden layers of size $D, 2D, 4D$ or $8D$, with RELU activation and residual connections. We vary the complexity of the MLP and variance of the initialization to create embeddings $g(y)$ to model a variety of settings of the extent of the staleness (${\WC(\DCal_Y, \tilde{\DCal}_Y)}$).

\textbf{Corrector Network} In these experiments, we vary the parameter count of the corrector network $h$ and number of hidden layers, using between $0$ and $2$ hidden layers with hidden dimension of $D, 2D, 4D,$ or $8D$. We use ReLU nonlinearity with residual connections. We optimize the parameters of the corrector network using Adam with learning rate $0.03$, and stop when the loss has not improved for at least $100$ epochs or we reach $1000$ epochs of training.

\textbf{Varying $|S(\alltargets)|$, number of targets used for training} In Figure~\ref{fig:sample_size}, we explore trade-offs between the complexity in terms of the parameter count $|\Psi|$ of $h$ (x-axis); the approximation error $\operatorname{KL}(P \Vert P_\approxmodel)$ after applying the trained correction model (y-axis); and the fraction of samples used for training $\approxmodel$. We report the complexity of the transformation from $\staletargetencoder$ to $\targetencoder$ in terms of  $\operatorname{KL}(P \Vert P_{\staletargetencoder})$ above each pane. 
Using a higher fraction of training samples is needed when there is more staleness. When the drift is more significant (right-hand pane), we observe that using increased parameters with a smaller fraction of samples does lead to overfitting. In this setting, it seems that sampling $10\%$ of the targets is generally sufficient.

\textbf{Varying Complexity of the Target Corrector Network} 
In order to explore how the KL divergence of our approximation may change with respect to the staleness of the embeddings $\staletargetencoder$, we train our embedding model to approximate the distributions $P$.
In Figure \ref{fig:param_count}, we explore how the KL divergence of our approximation may change with respect to the staleness of the embeddings $\staletargetencoder$,
We can obtain a significant reduction in KL divergence via the correction model (on the y-axis) across a wide variety of drifts (as measured by $\operatorname{KL}(P \Vert P_{\staletargetencoder})$). Increasing parameter count is always effective, but it yields greater benefit when approximating a distribution with greater divergence. 
\vspace{-2mm}
\section{Related work}

\textbf{Energy-based Models\hspace{1mm}} Many similar ideas of training small parametric models to aid the training of other models has been widely studied in energy-based models, such as CoopNets \cite{xie2018cooperative}, VERA \cite{grathwohl2021mcmc}, and others \cite{grathwohl2020learning}. In this setting models can be trained to skirt around intractable computations required in main-model training. 

\textbf{Amortized Inference\hspace{1mm}} There are many approaches that speed up sampling by fitting parametric models such as feed-forward neural networks \cite{marino2018iterative,naderiparizi2022amortized}. 

\textbf{Softmax Approximations\hspace{1mm}} Previous work has considered approximations to softmax via kernel methods \cite{blanc2018adaptive,rawat2019sampled} when there are trainable parameters for every target (rather than an encoder). Sampling-based approaches are widely used as well \cite{vembu2009probabilistic,zaheer2017canopy, monath2023improving}.

\textbf{Adapters\hspace{1mm}} Adapter methods, which train small parametric components of larger networks \cite{houlsby2019parameter} bear resemblance to our approach. 
However, our approach is distinct in that it operates only on the output layer of the neural models, not intermediate layers.
\section{Conclusion}
We present target corrector networks for approximating the softmax function during the training of dual encoder models and retrieval augmented language models. The target corrector networks learn to update a stale buffer of target representations.  We investigate the generalization properties of the corrector models theoretically. We furthermore show empirically how our correct model approach can be used to train models (both supervised retrievers and retrieval augmented language models) matching the accuracy of models that use 4x-80x the computational budget during training.

\section*{Impact Statement}

Our work proposes new more efficient ways of training of retrieval models. 
Retrieval models both in their own right and in combination with language models
have wide and applicable uses. The techniques of this paper are about improving
training efficiency. As such, better models could be produced faster, bringing to bear
all the responsibilities of model creators in terms of understanding the successes,
limitations, and biases of the model. Future work could consider the question of 
how different training strategies affect the way in which retrieval models
have broad impact. Of particular interest to this paper could be the way in
which staleness when computing the truncated softmax plays a role in such a study.

\bibliography{references}
\bibliographystyle{icml2024}

\newpage
\appendix
\onecolumn
\section{Analysis: Proofs}
\label{appendix:proofs}
\gapTprSpr*
\emph{Proof}. We bound the gap between true population risk and stale population risk. 
Recall that  $\GC_{\ell,  \FC}$ is the induced function class: $\GC_{\ell,  \FC} = \{y \mapsto  \ell(f(y), g(y)) : f \in \FC\}.$
Now note that 
\begin{align}
\label{eq:dist-shift}
&R_{\ell, f}(\DCal_Y) - R_{\ell, f}(\tilde{\DCal}_Y) \\
&= \mathbb{E}_{Y \sim \DCal_{Y}}\left[\ell(f(Y), g(Y)\right] - \mathbb{E}_{Y' \sim \tilde{\DCal}_{Y}}\left[\ell(f(Y'), g(Y')\right] \nonumber \\
& \leq \sup_{\lambda \in \GC_{\ell,  \FC}} \left(\mathbb{E}_{Y \sim \DCal_{Y}}\left[\lambda(X)\right] - \mathbb{E}_{Y' \sim \tilde{\DCal}_{Y}}\left[\lambda(Y')\right] \right) \nonumber \\
&\overset{(i)}{=} L \cdot \sup_{\lambda \in \GC_{\ell, \FC}} \left(\mathbb{E}_{Y \sim \DCal_{Y}}\left[\frac{\lambda(Y)}{L}\right] - \mathbb{E}_{Y' \sim \tilde{\DCal}_{Y}}\left[\frac{\lambda(Y')}{L}\right] \right) \nonumber \\
&\overset{(ii)}{\leq} L \cdot \sup_{\lambda \in {\rm Lip}_1(\rho)} \left(\mathbb{E}_{X \sim \DCal_{Y}}\left[\lambda(X)\right] - \mathbb{E}_{Y' \sim \tilde{\DCal}_{Y}}\left[\lambda(Y')\right] \right), \nonumber \\
&\overset{(iii)}{=}\WC(\DCal_Y, \tilde{\DCal}_Y)\\
&\overset{(iv)}{\leq}\|\DCal_Y - \tilde{\DCal}_Y\|_{TV}\\
&\overset{(v)}{=} \frac12 \sum_{y\in\mathcal{Y}} |\mathsf{softmax}(g(y)) - \mathsf{softmax}(g'(y))|\\
&\overset{(vi)}{\leq} \frac12 \|g - g'\|_1
\end{align}
where $(i)$ follows by dividing and multiply by $L$; $(ii)$ follows as, for any $\lambda \in \GC^h_{\ell,  \FC}$, we have $\frac{\lambda}{L}$ to be $1$-Lipschitz; $(iii)$ follows from Kantorovich-Rubinstein duality~\citep{villani2008optimal}; $(iv)$ follows from Corollory 6.14 in \citet{villani2008optimal}; $(v)$ follows from definition; and $(vi)$ follows from softmax Lipschtiz constant being 1. 
As $g$ and $g'$ are output from the same neural network but with parameters perturbed by $u$, then it follows that $\|g-g'\|_1 \leq L\|u\|$. $\qed$

\staleToEmpirical*

\emph{Proof} We need to connect the stale population risk to the empirical risk we are actually minimizing:
\begin{align}
\label{eq:pop-to-emp}
R_{\ell, \tilde{f}_n}(\tilde{\DCal}_Y)&=\mathbb{E}_{\tilde{\DCal}_{Y}}[\ell(\tilde{f}_n(Y), g(Y))]  \nonumber \\
&\leq \mathbb{E}_{\tilde{\SCal}_n}[\ell(\tilde{f}_n(Y), g(Y))] + \sup_{f \in \FC} \big|  \mathbb{E}_{\tilde{\DCal}_{Y}}[\ell(f(Y), g(Y))] - \mathbb{E}_{\tilde{\SCal}_n}[\ell(f(Y), g(Y))] \big| \nonumber \\
&\overset{(i)}{=} \mathbb{E}_{\tilde{\SCal}_n}[\ell(\tilde{f}_n(Y), h(X))] + \sup_{g \in \GC_{\ell, 
\FC}} \Big|\mathbb{E}_{\tilde{\DCal}_Y}[g(Y)] - \mathbb{E}_{\tilde{\SCal}_n}[g(Y)] \Big| \nonumber \\
&\overset{(ii)}{\leq}  \mathbb{E}_{\tilde{\SCal}_n}[\ell(\tilde{f}_n(Y), h(X))] + {\RF}_{\tilde{\SCal}_n}(\GC_{\ell,  \FC}) \nonumber \\
&= R_{\ell,\tilde{f}_n}(\tilde{\SCal}_n) + {\RF}_{\tilde{\SCal}_n}(\GC_{\ell,  \FC}),
\end{align}
where inequality $(i)$ follows from the definition of $\GC_{\ell,  \FC}$ and $(ii)$ from the standard symmetrization argument~~\citep{devroye2013probabilistic,mohri2018foundations} for Radamacher complexity. $\qed$

\maintheorem*

\begin{proof}
As  mentioned in the text $\FC$ might be too large function class and we would like to utilize the restricted function class $\FC^{g'}$. The previous derivation would go through using this restricted class and we will obtain the Rademacher complexity of ${\RF}_{\tilde{\SCal}_n}(\GC_{\ell,  \FC^{g'}})$ instead.
To compare the two Rademacher complexity, observe that
\begin{align}
{\RF}_{\tilde{\SCal}_n}(\GC_{\ell,  \FC^{g'}})
&= \frac{1}{n}\mathbb{E}_{\bm{\sigma}}\left\vert \sup_{\lambda \in \GC_{\ell,  \FC^{g'}}}  \sum_{i} \sigma_i \lambda(y_i) \right\vert \nonumber  \\
&\overset{(i)}{=} \frac{1}{n}\mathbb{E}_{\bm{\sigma}}\left\vert \sup_{h \in \HC}  \sum_{i} \sigma_i  \ell(h\circ g'(y), g(y)) \right\vert \nonumber \\
&= \frac{1}{n}\mathbb{E}_{\bm{\sigma}}\left\vert \sup_{f \in \FC^{g'}}  \sum_{i} \sigma_i  \ell(f(y), g(y)) \right\vert \nonumber \\
& \overset{(ii)}{\leq} \frac{1}{n}\mathbb{E}_{\bm{\sigma}}\left\vert \sup_{f \in \FC}  \sum_{i} \sigma_i \ell(f(y), g(y)) \right\vert \nonumber \\
&= \frac{1}{n}\mathbb{E}_{\bm{\sigma}}\left\vert \sup_{\lambda \in \GC_{\ell,  \FC}}  \sum_{i} \sigma_i \lambda(y_i) \right\vert \nonumber  \\
&= {\RF}_{\tilde{\SCal}_n}(\GC_{\ell,  \FC}),
\end{align}
where $(i)$ follows from definition of $\GC_{\ell,  \FC^{g'}}$ and $\FC^{g'}$; and $(ii)$ holds because $\FC^{g'} \subset \FC$.

Now, the standard concentration results for empirical Rademacher complexity implies that, with probability at least $1 - \delta$, we have the following.
\begin{align}
\label{eq:rade-conc}
{\RF}_{\tilde{\SCal}_n}(\GC_{\ell,  \FC^{g'}}) & \leq \mathbb{E}_{\tilde{\SCal}_n \sim \tilde{\DCal}^{\otimes n}_Y}\left[ \RF_{\tilde{\SCal}_n}(\GC_{\ell,  \FC^{g'}})\right] + \OC\Big(\sqrt{{\frac{\log(1/\delta)}{n}}}\Big) \\
&=
\tilde{\RF}_{n}(\GC_{\ell,  \FC^{g'}}) +  \OC\Big(\sqrt{{\frac{\log(1/\delta)}{n}}}\Big).
\end{align}

Combining results from Eq.~\ref{eq:dist-shift}, \ref{eq:pop-to-emp}, and \ref{eq:rade-conc}, we obtain that with probability at least $1-\delta$,
\begin{equation}
    R_{\ell, f}(\DCal_Y) \leq R_{\ell,\tilde{f}_n}(\tilde{\SCal}_n)
    + \underbrace{\WC(\DCal_Y, \tilde{\DCal}_Y)}_{\leq L \|u\|}
    + \tilde{\RF}_{n}(\GC_{\ell,  \FC^{g'}})
    + \OC\Big(\sqrt{{\frac{\log(1/\delta)}{n}}}\Big)
\end{equation}
\end{proof}

\clearpage
\begin{figure*}
    \centering
    \includegraphics[width=0.95\linewidth]{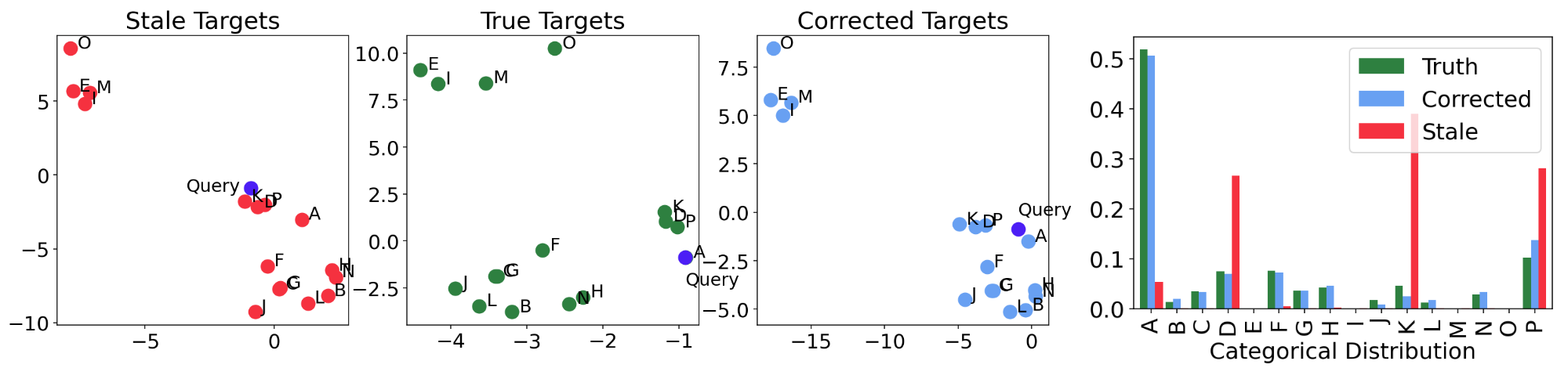}
    \caption{A toy experiment where stale and true targets are distributed around the unit circle, and the corrected targets based on the learned approximation are also depicted. The associated distribution over targets ($\beta=20$) based on the point identified.}
    \label{fig:distirbution shift}
\end{figure*}
\section{Experiments}
\label{appendix:experiments}

\subsection{Experimental Details}
\label{app:details}

\textbf{Training Details} We train all models, the dual-encoders and the corrector model, jointly using Adam \cite{kingma2014adam}. We implement the training procedure using stop-gradients so that the corrector model loss only changes the corrector model parameters and dual-encoder loss the dual-encoder ones. We form the subset of targets for the truncated softmax, $S(\alltargets)_\point$, using the top-64 closest targets to the given query according to a particular training procedure's buffer and 64 targets chosen uniformly at random. We use a minibatch size of 128 examples and share the truncated softmax targets across all examples in the minibatch $mb$ , e.g., $\bigcup_{\point \in mb} S(\alltargets)_\point$. We use 40K steps for retrieval training and 20K steps for RLM training. We combine the task losses and corrector network loss together. We experimented with a weight parameter applied to the corrector network. We use a weight value of 10.0.

\subsection{Additional Dense Retrieval Results}
\label{appendix:additionalresults}
In Table~\ref{tab:msmarco_t5}, we report performance on MSMarco using T5-base as the encoder. Here, with fewer targets, Stochastic Negative Mining provides a better approximation as a larger fraction of targets is re-encoded. Our method is still able to nearly match the performance of the exhaustive approach. We are able to achieve such results without having to re-embed the buffer.

\textbf{Using Accelerator Memory to Store the Buffer} In these experiments, we store the buffer of targets on the accelerator, making implementation of our approach training extremely easy. However, it could be the case that not all targets can fit into a buffer on accelerator memory. In such settings, our approach could still be used in the following ways: (1) subsample targets randomly (perhaps changing the subset periodically) to fit on device memory akin to a combination of our corrector approach stochastic negative mining, which would require no re-encoding of targets, or (2) use our approach to re-rank stale representations initially retrieved from CPU memory. 

\begin{table}
    \centering
    \begin{tabular}{@{}l r r r r r@{}}
    \toprule
         Performance & Num. Re-Embed & R@1 & R@5 & R@10 & R@100 \\
         \midrule
        Stale & 0 & 10.11 & 27.70 & 36.33 & 63.69 \\
        SNM & 20M & 18.18 & 43.48 & 54.68 & 82.18   \\
        Dynnibal & 8M & 18.23 & 43.15 & 54.56 & 82.24 \\
        Target Corrector & 0 & 17.07 & 40.78 & 51.56 & 79.29 \\
        Exhaustive & 352M & 18.18 & 44.97 & 55.58 & 83.69\\
    \bottomrule
    \end{tabular}
    \caption{\textbf{MSMARCO, T5 Initialization.} In this experiment, we measure performance of methods on MSMARCO. With fewer targets than Natural Questions, the gap between stochastic negative mining and the refreshed index is reduced.}
    \label{tab:msmarco_t5}
\end{table}

\paragraph{Comparisons to 2-Round Training}
Several recent works such as \cite{qu2021rocketqa} which addresses difficulties of training dense retrieval models proposes to train in 2 stages. First all targets are encoded (using random or pre-trained model). Then the model is trained for one half of the desired iterations. Then the new model's parameters are used to re-encode the targets a single time. Then the model is trained for the remaining steps using these re-encoded targets. We compare this approch with corrector networks in Table \ref{tab:2roundcomp}. We see that when using GTR-base, the performance for all methods is quite similar (with corrector networks and exhaustive re-encoding slightly outperforming). When T5-base is used though, we find the performance of corrector networks and exhaustive re-encoding to notably out-perform the 2-step procedure. We attribute this to GTR being a better initialization for the model. In this case we would expect its parameters (and therefore its target embeddings) to change less from pre-training to fine-tuning, meaning that there is less embedding drift and therefore less bias when using the 2-step procedure. 

\begin{table}
    \centering
    \begin{tabular}{@{}l r r r r r r@{}}
    \toprule
         
Method	&Base	&R@1	&R@5	&R@10	&R@20 &R@100\\
         \midrule
       Two Round&	T5&	29.50&	53.40&	62.49&	70.64&	80.94 \\
        Corrector&	T5&	36.65&	59.25&	68.06&	73.71&	83.13  \\
        Exhaustive&	T5&	37.34&	60.42&	68.70&	74.76&	83.41\\
        \midrule
        Two Round&	GTR&	49.06&	70.06&	76.76	&81.17	&87.95\\
Corrector&	GTR	&49.61&	70.72&	77.04&	82.33&
88.28\\
Exhaustive&	GTR&	50.30&	71.55&	78.12&	82.83&	88.59\\
    \bottomrule
    \end{tabular}
    \caption{\textbf{Comparison with 2-round training.} We compare the performance of corrector networks, exhaustive re-encoding, and 2-round training. Results are presented on the NQ Test set.}
    \label{tab:2roundcomp}
\end{table}

\paragraph{Comparisons with and without uniform negatives} In our main experiments (as stated in Appendix \ref{app:details}) we train with hard negatives \emph{and} uniform negatives. Initial experiments showed that adding uniform negatives lead to improved performance in some settings. We provide some additional results ablating this choice using exhaustive re-encoding. These can be found in Table \ref{tab:unifnegcomp}. We can see that this choice provides negligible improvement on the reported benchmarks (although we believe its worth trying in other settings). 

\begin{table}
    \centering
    \begin{tabular}{@{}l r r r@{}}
    \toprule
         Method & Base &  R@1 & R@100 \\
         \midrule
        With Uniform & T5	&24.70&	79.94 \\
        Without Uniform& T5&	24.87	&79.82  \\
        \midrule
       With Uniform	& GTR & 36.29&	91.73 \\
       Without Uniform & GTR&	36.53&	91.70\\
    \bottomrule
    \end{tabular}
    \caption{\textbf{Uniform Negatives Comparison.} We compare the performance of exhaustive re-encoding with and without the addition of uniform (in-batch) negatives. Results are presented on the NQ Dev set.}
    \label{tab:unifnegcomp}
\end{table}

\subsection{Retrieve and Read}
Note that in this setting we do not share the subset of targets $S(\alltargets)$ across the examples in the batch, nor do we use targets sampled uniformly at random.

The versions of the retrieve-and read datasets are:
\begin{itemize}
    \item  TriviaQA: {\url{https://www.tensorflow.org/datasets/catalog/trivia_qa#trivia_qaunfilterednocontext}}
    \item NQOpen \url{https://www.tensorflow.org/datasets/catalog/natural_questions_open}
    \item HotPotQA \url{https://www.tensorflow.org/datasets/community_catalog/huggingface/hotpot_qa}
\end{itemize}

\clearpage
\subsection{Beyond Stale Representations: Approximating Large Models with Small Models}

In this experiment, we focus on sampling in isolation. We sample a batch of input points and we measure the ability of our method to approximate one dual encoder model with another. In particular, we study a case where we approximate a large dual encoder with a small model. We approximate the GTR large model \cite{ni2021large} (e.g., $\targetencoder(\cdot)$) with the GTR small model(e.g., $\staletargetencoder(\cdot)$). In Table~\ref{tab:small_large},  we report nearest neighbor precision, i.e., measuring the overlap in the top-K neighbors from the large model's neighbors at 10, 20, and 100 on the dataset Arguana \cite{wachsmuth2018retrieval} from the BEIR benchmark \cite{thakur2021beir}. We use 32 samples for each query to train the correction model. We find that overlap amongst smaller K seems to be better aligned using our method.

\begin{table}
    \centering
    \begin{tabular}{@{}l r r r r@{}}
    \toprule
         Performance & P@10 & P@20 & P@100  \\
         \midrule
        $\staletargetencoder(\cdot)$ & 67.57  & 67.73 & 57.57 \\
        $\approxmodel \circ \staletargetencoder(\cdot)$ & 76.87 & 73.32 & 53.55 \\
    \bottomrule
    \end{tabular}
    \caption{\textbf{Approximating Large Models with Small Models} On the dataset Arguana \cite{wachsmuth2018retrieval}, we use our method to warp the embedding space of GTR Small so that it is better aligned with GTR Large. Note that here we present nearest neighbor precision, i.e., the overlap in the top-K neighbors from the large model at 10, 20, and 100. We use 32 samples for each query to train the correction model.}
    \label{tab:small_large}
\end{table}


\end{document}